\documentclass{article}

\newif\ifsup
\newif\ifnips
\nipsfalse
\suptrue

\ifnips
\usepackage{nips15submit_e}
\else
\fi
\usepackage{url}

\usepackage[disable]{todonotes}

\newcommand{\todot}[1]{\todo[inline,color=blue!20!white]{\scriptsize #1}}

\usepackage{latexsym}
\usepackage{amsmath}
\usepackage{amssymb}
\usepackage{mathtools}
\usepackage{enumerate}
\usepackage{accents}
\usepackage{tikz}
\ifnips
\usetikzlibrary{external}
\fi

\usepackage{pgfplots}
\usepackage[boxed]{algorithm}
\usepackage{algpseudocode}

\usepackage{nicefrac}
\usepackage{dsfont}
\usepackage[bf]{caption}
\ifnips
\usepackage{hyperref}
\hypersetup{
    bookmarks=true,         
    unicode=false,          
    pdftoolbar=true,        
    pdfmenubar=true,        
    pdffitwindow=false,     
    pdfstartview={FitH},    
    pdftitle={My title},    
    pdfauthor={Author},     
    pdfsubject={Subject},   
    pdfcreator={Creator},   
    pdfproducer={Producer}, 
    pdfkeywords={keyword1} {key2} {key3}, 
    pdfnewwindow=true,      
    colorlinks=true,       
    linkcolor=red,          
    citecolor=blue,        
    filecolor=magenta,      
    urlcolor=cyan           
}
\else
\usepackage[colorlinks=true,citecolor=blue]{hyperref}
\fi
\usepackage{amsthm}
\usepackage{times}
\usepackage{natbib}
\usepackage{wrapfig}
\usepackage{enumitem}
\usepackage[capitalize]{cleveref}

\newcommand{\E}{\mathbb E}

\newcommand{\calF}{\mathcal F}

\newcommand{\calB}{\mathcal B}
\newcommand{\sr}[1]{\stackrel{#1}}
\newcommand{\set}[1]{\left\{#1\right\}}
\newcommand{\ind}[1]{\mathds{1}\!\set{#1}}
\newcommand{\argmax}{\operatornamewithlimits{arg\,max}}

\newcommand{\ceil}[1]{\left \lceil {#1} \right\rceil}
\newcommand{\eqn}[1]{\begin{align}#1\end{align}}

\newcommand{\eq}[1]{\begin{align*}#1\end{align*}}
\newcommand{\frontier}{\ensuremath{\delta\calB}}

\renewcommand{\P}[1]{\mathbb{P}\left\{#1\right\}}

\newcommand{\KL}{\operatorname{KL}}
\newcommand{\plog}{\operatorname{ProductLog}}
\newcommand{\logp}{\log_+\!\!}

\newcommand{\R}{\mathbb R}

\let\epsilon\varepsilon

\theoremstyle{plain}
\newtheorem{theorem}{Theorem}

\newtheorem{lemma}[theorem]{Lemma}
\newtheorem{corollary}[theorem]{Corollary}
\theoremstyle{definition}

\theoremstyle{remark}

\ifnips
\title{
The Pareto Regret Frontier for Bandits
}
\date{}
\author{Tor Lattimore \\
Department of Computing Science \\ University of Alberta, Canada \\ \texttt{tor.lattimore@gmail.com}
}

\else
\title{\vskip 2mm\bf\LARGE\hrule height3pt \vskip 6mm
The Pareto Regret Frontier for Bandits
\vskip 6mm \hrule height1pt \vskip 5mm}
\date{}
\author{\bf Tor Lattimore \\
Department of Computing Science \\
University of Alberta, Canada \\ 
\texttt{tor.lattimore@gmail.com}
}
\fi

\ifnips
\nipsfinalcopy
\fi

\begin{document}

\maketitle

\begin{abstract}
Given a multi-armed bandit problem it may be desirable to achieve a smaller-than-usual worst-case regret for some special actions.
I show that the price for such unbalanced worst-case regret guarantees is rather high. Specifically, if an algorithm enjoys a worst-case
regret of $B$ with respect to some action, then there must exist another action for which the worst-case regret is at least $\Omega(nK/B)$, where $n$ is
the horizon and $K$ the number of actions. I also give upper bounds in both the stochastic and adversarial settings showing that this result cannot
be improved. For the stochastic case the pareto regret frontier is characterised exactly up to constant factors.
\end{abstract}

\section{Introduction}

The multi-armed bandit is the simplest class of problems that exhibit the exploration/exploitation dilemma.
In each time step the learner chooses one of $K$ actions and receives a noisy reward signal for the chosen action. 
A learner's performance is measured in terms of the regret, which is the (expected) difference between the rewards it actually received
and those it would have received (in expectation) by choosing the optimal action.

Prior work on the regret criterion for finite-armed bandits has treated all actions uniformly and has aimed for bounds on the regret that
do not depend on which action turned out to be optimal. I take a different approach and ask what can be achieved if some actions are given
special treatment. Focussing on worst-case bounds, I ask whether or not it is possible to achieve improved worst-case regret for some
actions, and what is the cost in terms of the regret for the remaining actions. Such results may be useful in a variety of cases. For example,
a company that is exploring some new strategies might expect an especially small regret if its existing strategy turns out to be (nearly) optimal.

This problem has previously been considered in the experts setting where the learner is allowed to observe the reward for all actions in every round, not only for the
action actually chosen. The earliest work seems to be by \cite{HP05} where it is shown that the learner can assign a prior weight to each action and pays
a worst-case regret of $O(\sqrt{-n \log \rho_i})$ for expert $i$ where $\rho_i$ is the prior belief in expert $i$ and $n$ is the horizon. The uniform
regret is obtained by choosing $\rho_i = 1/K$, which leads to the well-known $O(\sqrt{n \log K})$ bound achieved by the exponential weighting algorithm \citep{Ces06}.
The consequence of this is that an algorithm can enjoy a constant regret with respect to a single action while suffering minimally on the remainder.
The problem was studied in more detail by \cite{Koo13} where (remarkably) the author was able to {\it exactly} describe the pareto regret frontier when $K=2$. 

Other related work (also in the experts setting) is where the objective is to obtain an improved regret against a mixture of available experts/actions
\citep{EKMW08,KP11}. In a similar vain, \cite{SNL14} showed that algorithms for prediction with expert advice can be combined with minimal cost to obtain the
best of both worlds. In the bandit setting I am only aware of the work by \cite{LL15} who study the effect of the prior on the regret of Thompson sampling in
a special case. In contrast the lower bound given here applies to all algorithms in a relatively standard setting.

The main contribution of this work is a characterisation of the pareto regret frontier (the set of achievable worst-case regret bounds) 
for stochastic bandits. 

Let $\mu_i \in \R$ be the unknown mean of the $i$th arm and assume that $\sup_{i,j} \mu_i - \mu_j \leq 1$. In each time step the learner chooses an
action $I_t \in \set{1,\ldots, K}$ and receives reward $g_{I_t,t} = \mu_i + \eta_t$ where $\eta_t$ is the noise term that I assume to be 
sampled independently from a $1$-subgaussian distribution that may depend on $I_t$. This model subsumes both Gaussian and Bernoulli (or bounded)
rewards. Let $\pi$ be a bandit strategy, which is a function from histories of observations to an action $I_t$. Then the $n$-step 
expected pseudo regret with respect to the $i$th arm is
\eq{
R^\pi_{\mu,i} = n \mu_i - \E \sum_{t=1}^n \mu_{I_t}\,,
}
where the expectation is taken with respect to the randomness in the noise and the actions of the policy.
Throughout this work $n$ will be fixed, so is omitted from the notation.
The worst-case expected pseudo-regret with respect to arm $i$ is 
\eqn{
\label{eq:wcregret}
R^\pi_i = \sup_\mu R^\pi_{\mu, i}\,.
}
This means that $R^\pi \in \R^K$ is a vector of worst-case pseudo regrets with respect to each of the arms.
Let $\calB \subset \R^K$ be a set defined by
\eqn{
\label{def:B}
\calB = \set{B \in [0,n]^K : B_i \geq \min\set{n,\, \sum_{j \neq i} \frac{n}{B_j}} \text{ for all } i}\,.
}
The boundary of $\calB$ is denoted by $\frontier$.
The following theorem shows that $\frontier$ describes the pareto regret frontier up to constant factors. \\
\vspace{-0.2cm}
\begin{center}
\begin{tikzpicture}
\node[draw,text width=12.5cm,inner sep=5pt] at (0,0) {
\vspace{-0.2cm}
\begin{center}
\bf Theorem \\
\end{center}
There exist universal constants $c_1 = 8$ and $c_2 = 252$ such that: \\[0.3cm]
{\bf Lower bound:} for $\eta_t \sim \mathcal N(0,1)$ and all strategies $\pi$ we have $c_1(R^\pi + K) \in \calB$ \\[0.2cm]
{\bf Upper bound:} for all $B \in \calB$ there exists a strategy $\pi$ such that $R^\pi_i \leq c_2 B_i$ for all $i$
};
\end{tikzpicture}
\end{center}
Observe that the lower bound relies on the assumption that the noise term be Gaussian while the upper bound holds for subgaussian
noise. The lower bound may be generalised to other noise models such as Bernoulli, but does not hold for all subgaussian noise models. For example,
it does not hold if there is no noise ($\eta_t = 0$ almost surely).

The lower bound also applies to the adversarial framework where the rewards may be chosen arbitrarily.
Although I was not able to derive a matching upper bound in this case, a simple modification of the Exp-$\gamma$ algorithm \citep{BC12} leads to an algorithm
with 
\eq{
R^\pi_1 \leq B_1 \quad\text{and}\quad R^\pi_k \lesssim \frac{nK}{B_1} \log\left(\frac{nK}{B_1^2}\right) \text{ for all } k \geq 2 \,,
}
where the regret is the adversarial version of the expected regret. 
\ifsup
The details may be found in the Appendix.
\fi

The new results seem elegant, but disappointing. In the experts setting we have seen that the learner can distribute a prior amongst the actions and
obtain a bound on the regret depending in a natural way on the prior weight of the optimal action. In contrast, in the bandit setting the learner 
pays an enormously higher price to obtain a small regret with respect to even a single arm. In fact, the learner must essentially choose a single arm to favour,
after which the regret for the remaining arms has very limited flexibility. Unlike in the experts setting, if even a single arm enjoys constant worst-case 
regret, then the worst-case regret with respect to all other arms is necessarily linear.

\section{Preliminaries}\label{sec:prel}

I use the same notation as \cite{BC12}.
Define $T_i(t)$ to be the number of times action $i$ has been chosen after time step $t$ and $\hat \mu_{i,s}$ to be the empirical estimate
of $\mu_i$ from the first $s$ times action $i$ was sampled. This means that $\hat \mu_{i,T_i(t-1)}$ is the empirical estimate of $\mu_i$ at the start of the $t$th round.
I use the convention that $\hat \mu_{i,0} = 0$.
Since the noise model is $1$-subgaussian we have
\eqn{
\label{eq:subgaussian}
\forall \epsilon > 0 \qquad \P{\exists s \leq t : \hat \mu_{i,s} - \mu_i \geq \epsilon / s} \leq \exp\left(-\frac{\epsilon^2}{2t} \right)\,.
}
\ifsup
This result is presumably well known, but a proof is included in \cref{app:conc} for convenience.
\fi
The optimal arm is $i^* = \argmax_i \mu_i$ with ties broken in some arbitrary way. The optimal reward is $\mu^* = \max_i \mu_i$.
The gap between the mean rewards of the $j$th arm and the optimal arm is 
$\Delta_j = \mu^* - \mu_j$ and $\Delta_{ji} = \mu_i - \mu_j$.
The vector of worst-case regrets is $R^\pi \in \R^K$ and has been defined already in \cref{eq:wcregret}. 
I write $R^\pi \leq B \in \R^K$ if $R^\pi_i \leq B_i$ for all $i \in \set{1,\ldots,K}$.
For vector $R^\pi$ and $x \in \R$ we have $(R^\pi + x)_i = R^\pi_i + x$. 

\section{Understanding the Frontier}\label{sec:frontier}

Before proving the main theorem I briefly describe the features of the regret frontier.
First notice that if $B_i = \sqrt{n(K-1)}$ for all $i$, then
\eq{
B_i 
= \sqrt{n(K-1)} 
= \sum_{j\neq i} \sqrt{n/(K-1)} 
= \sum_{j\neq i} \frac{n}{B_j}\,.
}
Thus $B \in \calB$ as expected. This particular $B$ is witnessed up to constant factors by MOSS \citep{AB09} and OC-UCB \citep{Lat15-ucb}, but
not UCB \citep{ACF02}, which suffers $R^\text{ucb}_i \in \Omega(\sqrt{n K \log n})$.

Of course the uniform choice of $B$ is not the only option. Suppose the first arm is special, so $B_1$ should be chosen especially small.
Assume without loss of generality that $B_1 \leq B_2 \leq \ldots \leq B_K \leq n$.
Then by the main theorem we have
\eq{
B_1 \geq \sum_{i=2}^K \frac{n}{B_i} \geq \sum_{i=2}^k \frac{n}{B_i} \geq \frac{(k-1)n}{B_k}\,.
}
Therefore
\eqn{
\label{eq:simple-lower}
B_k \geq \frac{(k-1)n}{B_1}\,.
}
This also proves the claim in the abstract, since it implies that $B_K \geq (K-1)n / B_1$.
If $B_1$ is fixed, then choosing $B_k = (k-1)n / B_1$ does not lie on the frontier because
\eq{
\sum_{k=2}^K \frac{n}{B_k} = \sum_{k=2}^K \frac{B_1}{k-1} \in \Omega(B_1 \log K)
}
However, if $H = \sum_{k=2}^K 1/(k-1) \in \Theta(\log K)$, then choosing $B_k = (k-1)nH / B_1$ 
does lie on the frontier and is a factor of $\log K$ away from the lower bound given in \cref{eq:simple-lower}. 
Therefore up the a $\log K$ factor, points on the regret frontier are characterised entirely by a permutation determining
the order of worst-case regrets and the smallest worst-case regret.

Perhaps the most natural choice of $B$ (assuming again that $B_1 \leq\ldots \leq B_K$) is
\eq{
B_1 = n^p \qquad \text{and} \qquad B_k = (k-1)n^{1-p}H \text{ for } k > 1\,.
}
For $p = 1/2$ this leads to a bound that is at most $\sqrt{K} \log K$ worse than that obtained by MOSS and OC-UCB while being a factor of $\sqrt{K}$ better for 
a select few.

\subsubsection*{Assumptions}
The assumption that $\Delta_i \in [0,1]$ is used to avoid annoying boundary problems caused by the fact that time is discrete. This means that if $\Delta_i$ is extremely large, then
even a single sample from this arm can cause a big regret bound. This assumption is already quite common, for example a worst-case regret of $\Omega(\sqrt{Kn})$ clearly does not
hold if the gaps are permitted to be unbounded. Unfortunately there is no perfect resolution to this annoyance. Most elegant would be to allow time to be continuous with actions
taken up to stopping times. Otherwise you have to deal with the discretisation/boundary problem with special cases, or make assumptions as I have done here.

\section{Lower Bounds}\label{sec:lower}

\begin{theorem}\label{thm:lower}
Assume $\eta_t \sim \mathcal N(0, 1)$ is sampled from a standard Gaussian.
Let $\pi$ be an arbitrary strategy, then $\displaystyle 8(R^\pi + K) \in \calB$.
\end{theorem}

\begin{proof}
Assume without loss of generality that $R^\pi_1 = \min_i R^\pi_i$ (if this is not the case, then simply re-order the actions).
If $R^\pi_1 > n/8$, then the result is trivial. From now on assume $R^\pi_1 \leq n/8$.
Let $c = 4$ and define
\eq{
\epsilon_k = \min\set{\frac{1}{2}, \frac{c R^\pi_k}{n}} \leq \frac{1}{2}\,.
}
Define $K$ vectors $\mu_1,\ldots, \mu_K \in \R^K$ by
\eq{
(\mu_k)_j = \frac{1}{2} + \begin{cases}
0 & \text{if } j = 1 \\
\epsilon_k & \text{if } j = k \neq 1 \\
-\epsilon_j & \text{otherwise}\,.
\end{cases}
}
Therefore the optimal action for the bandit with means $\mu_k$ is $k$.
Let $A = \set{k : R^\pi_k \leq n / 8}$ and $A' = \set{k : k \notin A}$ and assume $k \in A$. Then
\eq{
R^\pi_k 
\sr{(a)}\geq R^\pi_{\mu_k,k} 
\sr{(b)}\geq \epsilon_k \E^\pi_{\mu_k} \left[\sum_{j \neq k} T_j(n)\right] 
\sr{(c)}= \epsilon_k \left(n - \E^\pi_{\mu_k} T_k(n)\right)  
\sr{(d)}= \frac{c R^\pi_k (n - \E^\pi_{\mu_k} T_k(n))}{n} \,,
}
where (a) follows since $R^\pi_k$ is the worst-case regret with respect to arm $k$,
(b) since the gap between the means of the $k$th arm and any other arm is at least $\epsilon_k$ (Note that this
is also true for $k = 1$ since $\epsilon_1 = \min_k \epsilon_k$.
(c) follows from the fact that $\sum_i T_i(n) = n$ and
(d) from the definition of $\epsilon_k$.
Therefore
\eqn{
\label{eq:linear}
n\left(1 - \frac{1}{c}\right)
\leq \E^\pi_{\mu_k} T_k(n)\,.
}
Therefore for $k \neq 1$ with $k \in A$ we have
\eq{
n\left(1 - \frac{1}{c}\right)
&\leq \E^\pi_{\mu_k} T_k(n) 
\sr{(a)}\leq \E^\pi_{\mu_1} T_k(n) + n\epsilon_k \sqrt{\E^\pi_{\mu_1} T_k(n)} \\
&\sr{(b)}\leq n - \E^\pi_{\mu_1} T_1(n) + n\epsilon_k \sqrt{\E^\pi_{\mu_1} T_k(n)} 
\sr{(c)}\leq \frac{n}{c} + n \epsilon_k \sqrt{\E^\pi_{\mu_1} T_k(n)}\,,
}
where (a) follows from standard entropy inequalities and a similar argument as used by \cite{ACFS95} 
\ifsup
(details given in \cref{app:kl}),
\fi
(b) since $k \neq 1$ and $\E^\pi_{\mu_1} T_1(n) + \E^\pi_{\mu_1} T_k(n) \leq n$, and
(c) by \cref{eq:linear}.
Therefore
\eq{
\E^\pi_{\mu_1} T_k(n) 
\geq \frac{1 - \frac{2}{c}}{\epsilon_k^2}\,,
}
which implies that
\eq{
R^\pi_1 
\geq R^\pi_{\mu_1,1} 
= \sum_{k=2}^K \epsilon_k \E^\pi_{\mu_1} T_k(n) 
\geq \sum_{k \in A - \set{1}} \frac{1 - \frac{2}{c}}{\epsilon_k} 
= \frac{1}{8} \sum_{k \in A - \set{1}} \frac{n}{R^\pi_k}\,.
}
Therefore for all $i \in A$ we have
\eq{
8R^\pi_i \geq \sum_{k \in A - \set{1}} \frac{n}{R^\pi_k} \cdot \frac{R^\pi_i}{R^\pi_1} 
\geq \sum_{k \in A - \set{i}} \frac{n}{R^\pi_k}\,. 
}
Therefore
\eq{
8R^\pi_i + 8K \geq \sum_{k\neq i} \frac{n}{R^\pi_k} + 8K - \sum_{k \in A' - \set{i}} \frac{n}{R^\pi_k} 
\geq \sum_{k \neq i} \frac{n}{R^\pi_k}\,,
}
which implies that $8(R^\pi + K) \in \calB$ as required.
\end{proof}

\section{Upper Bounds}\label{sec:upper}

I now show that the lower bound derived in the previous section is tight up to constant factors.
The algorithm is a generalisation MOSS \citep{AB09} with two modifications.
First, the width of the confidence bounds are biased in a non-uniform way, and second, the upper confidence bounds
are shifted. 
The new algorithm is functionally identical to MOSS in the special case that $B_i$ is uniform.
Define $\log_{+}\!(x) = \max\set{0, \log(x)}$.

\begin{center}
\begin{minipage}{11cm}
\begin{algorithm}[H]
\caption{Unbalanced MOSS}\label{alg:moss}
\begin{algorithmic}[1]
\State {\bf Input:} $n$ and $B_1,\ldots,B_K$
\State $n_i = n^2 / B_i^2$ for all $i$
\For{$t \in 1,\ldots,n$}
\State $\displaystyle I_t = \argmax_i \hat \mu_{i,T_i(t-1)} + \sqrt{\frac{4}{T_i(t-1)} \logp\left(\frac{n_i}{T_i(t-1)}\right)} - \sqrt{\frac{1}{n_i}}$
\EndFor
\end{algorithmic}
\end{algorithm}
\end{minipage}
\end{center}
\vspace{0.2cm}

\begin{theorem}\label{thm:moss}
Let $B \in \calB$, then the strategy $\pi$ given in \cref{alg:moss} satisfies
$R^\pi \leq 252B$.
\end{theorem}

\begin{corollary}
For all $\mu$ the following hold:
\begin{enumerate}
\item $R^\pi_{\mu,i^*} \leq 252 B_{i^*}$.
\item $R^\pi_{\mu,i^*} \leq \min_i (n\Delta_i + 252 B_i)$
\end{enumerate}
\end{corollary}

The second part of the corollary is useful when $B_{i^*}$ is large, but there exists an arm
for which $n\Delta_i$ and $B_i$ are both small.
The proof of \cref{thm:moss} requires a few lemmas.
The first is a somewhat standard concentration inequality 
that follows from a combination of the peeling argument and Doob's maximal inequality.

\newcommand{\fixnj}{\vphantom{n_i}\smash[b]{1/n_j}}

\begin{lemma}\label{lem:peel}
Let $\displaystyle Z_i = \max_{1 \leq s \leq n} \mu_{i} - \hat \mu_{i,s} - \sqrt{\frac{4}{s} \logp\left(\frac{n_{i}}{s}\right)}$.
Then $\P{Z_i \geq \Delta} \leq \frac{20}{n_{i} \Delta^2}$ for all $\Delta > 0$.
\end{lemma}

\ifsup
\begin{proof}
Using the peeling device.
\eq{
\P{Z_i \geq \Delta} 
&\sr{(a)}= \P{\exists s \leq n : \mu_i - \hat \mu_{i,s} \geq \Delta + \sqrt{\frac{4}{s} \logp\left(\frac{n_{i}}{s}\right)} } \\
&\sr{(b)}\leq \sum_{k=0}^\infty \P{\exists s < 2^{k+1} : s(\mu_{i} - \hat \mu_{i,s}) \geq 2^k \Delta + \sqrt{2^{k+2} \logp\left(\frac{n_{i}}{2^{k+1}}\right)}} \\
&\sr{(c)}\leq \sum_{k=0}^\infty \exp\left(-2^{k-2} \Delta^2 \right) \min\set{1, \frac{2^{k+1}}{n_{i}}} 
\sr{(d)}\leq \left(\frac{8}{\log(2)} + 8\right) \cdot \frac{1}{n_{i} \Delta^2} 
\leq \frac{20}{n_{i}\Delta^2}\,,
}
where (a) is just the definition of $Z_i$,
(b) follows from the union bound and re-arranging the equation inside the probability,
(c) follows from \cref{eq:subgaussian} and the definition of $\log_{+}$ and
(d) is obtained by upper bounding the sum with an integral.
\end{proof}
\fi
In the analysis of traditional bandit algorithms the gap $\Delta_{ji}$ measures how quickly the algorithm can detect the
difference between arms $i$ and $j$. By design, however, \cref{alg:moss} is negatively biasing
its estimate of the empirical mean of arm $i$ by $\sqrt{1/n_i}$. This has the effect of shifting the gaps, which I denote
by $\bar \Delta_{ji}$ and define to be
\eq{
\bar \Delta_{ji} = \Delta_{ji} + \sqrt{\fixnj} - \sqrt{1/n_i}
= \mu_i - \mu_j + \sqrt{\fixnj} - \sqrt{1/n_i}\,.
}

\begin{lemma}\label{lem:stop}
Define stopping time $\tau_{ji}$ by
\eq{
\tau_{ji} = \min\set{s : \hat \mu_{j,s} + \sqrt{\frac{4}{s} \logp\left(\frac{n_j}{s}\right)} \leq \mu_j + \bar \Delta_{ji} / 2}\,.
}
If $Z_i < \bar \Delta_{ji} / 2$, then $T_j(n) \leq \tau_{ji}$. 
\end{lemma}

\begin{proof}
Let $t$ be the first time step such that $T_j(t-1) = \tau_{ji}$. Then
\eq{
\hat \mu_{j,T_j(t-1)} + &\sqrt{\frac{4}{T_j(t-1)} \logp\left(\frac{n_j}{T_j(t-1)}\right)} - \sqrt{\fixnj}
\leq \mu_j + \bar \Delta_{ji} / 2 - \sqrt{\fixnj} \\
&=\mu_j + \bar \Delta_{ji} - \bar \Delta_{ji} /2 - \sqrt{\fixnj} \\
&=\mu_i - \sqrt{1/n_i} - \bar \Delta_{ji}/2 \\
&<\hat \mu_{i,T_i(t-1)} + \sqrt{\frac{4}{T_i(t-1)} \logp\left(\frac{n_i}{T_i(t-1)}\right)} - \sqrt{1/n_i}\,,
}
which implies that arm $j$ will not be chosen at time step $t$ and so also not for any subsequent time steps by the same argument and induction.
Therefore $T_j(n) \leq \tau_{ji}$.
\end{proof}

\begin{lemma}\label{lem:eliminator}
If $\bar \Delta_{ji} > 0$,
then $\displaystyle \E \tau_{ji} \leq \frac{40}{\bar \Delta_{ji}^2} + \frac{64}{\bar \Delta_{ji}^2} \plog\left(\frac{n_j \bar\Delta_{ji}^2}{64}\right)$.
\end{lemma}

\begin{proof}
Let $s_0$ be defined by
\eq{
s_0 = \ceil{\frac{64}{\bar \Delta_{ji}^2} \plog\left(\frac{n_j\bar \Delta_{ji}^2}{64}\right)}
\quad\implies\quad
\sqrt{\frac{4}{s_0} \logp\left(\frac{n_j}{s_0}\right)} \leq \frac{\bar \Delta_{ji}}{4}\,.
}
Therefore
\eq{
\E \tau_{ji} 
&= \sum_{s=1}^n \P{\tau_{ji} \geq s} 
\leq 1 + \sum_{s=1}^{n-1} \P{\hat \mu_{i,s} - \mu_{i,s} \geq \frac{\bar\Delta_{ji}}{2} - \sqrt{\frac{4}{s} \logp\left(\frac{n_j}{s}\right)} } \\
&\leq 1 + s_0 + \sum_{s=s_0+1}^{n-1} \P{\hat \mu_{i,s} - \mu_{i,s} \geq \frac{\bar\Delta_{ji}}{4}} 
\leq 1 + s_0 + \sum_{s=s_0+1}^\infty \exp\left(-\frac{s \bar \Delta_{ji}^2}{32}\right) \\ 
&\leq 1 + s_0 + \frac{32}{\bar \Delta_{ji}^2} 
\leq \frac{40}{\bar \Delta_{ji}^2} + \frac{64}{\bar \Delta_{ji}^2} \plog \left(\frac{n_j \bar \Delta_{ji}^2}{64}\right)\,,
}
where the last inequality follows since $\bar \Delta_{ji} \leq 2$.
\end{proof}

\begin{proof}[Proof of \cref{thm:moss}]
Let $\Delta = 2/\sqrt{n_i}$ and $A = \set{j : \Delta_{ji} > \Delta}$. Then for $j \in A$ we have
$\Delta_{ji} \leq 2 \bar\Delta_{ji}$ and $\bar \Delta_{ji} \geq \sqrt{1/n_i} + \sqrt{\fixnj}$.
Letting $\Delta' = \sqrt{1/n_i}$ we have 
\eq{
R^\pi_{\mu,i} 
&= \E\left[\sum_{j=1}^K \Delta_{ji} T_j(n)\right] \\
&\leq n\Delta + \E\left[\sum_{j \in A} \Delta_{ji} T_j(n)\right] \\
&\sr{(a)}\leq 2B_i + \E\left[\sum_{j \in A} \Delta_{ji} \tau_{ji} + n \max_{j \in A} \set{\Delta_{ji} : Z_i \geq \bar \Delta_{ji}/2} \right] \\
&\sr{(b)}\leq 2B_i + \sum_{j \in A} \left(\frac{80}{\bar\Delta_{ji}} + \frac{128}{\bar\Delta_{ji}}\plog\left(\frac{n_j \bar\Delta_{ji}^2}{64}\right)\right) + 4n \E[Z_i \ind{Z_i \geq \Delta'}] \\ 
&\sr{(c)}\leq 2B_i + \sum_{j \in A} 90 \sqrt{n_j} + 4n \E[Z_i \ind{Z_i \geq \Delta'}]\,, 
}
where (a) follows by using \cref{lem:stop} to bound $T_j(n) \leq \tau_{ji}$ when $Z_i < \bar \Delta_{ji}$.
On the other hand, the total number of pulls for arms $j$ for which $Z_i \geq \bar \Delta_{ji} / 2$ is at most $n$.
(b) follows by bounding $\tau_{ji}$ in expectation using \cref{lem:eliminator}.
(c) follows from basic calculus and because for $j \in A$ we have $\bar \Delta_{ji} \geq \sqrt{1/n_i}$.
All that remains is to bound the expectation.
\eq{
4n \E[Z_i \ind{Z_i \geq \Delta'}]
&\leq 4n \Delta' \P{Z_i \geq \Delta'} + 4n \int^\infty_{\Delta'} \P{Z_i \geq z} dz 
\leq \frac{160n}{\Delta' n_i} 
= \frac{160n}{ \sqrt{n_i}} = 160B_i\,,
}
where I have used \cref{lem:peel} and simple identities.
Putting it together we obtain
\eq{
R^\pi_{\mu,i} \leq 2 B_i + \sum_{j \in A} 90 \sqrt{n_j} + 160B_1
&\leq 252 B_i\,,
}
where I applied the assumption $B \in \calB$ and so $\sum_{j \neq 1} \sqrt{n_j} = \sum_{j \neq 1} n/B_j \leq B_i$.
\end{proof}

The above proof may be simplified in the special case that $B$ is uniform where we recover the minimax regret of MOSS, but with perhaps a simpler proof
than was given originally by \cite{AB09}.

\subsection*{On Logarithmic Regret}

In a recent technical report I demonstrated empirically that MOSS suffers sub-optimal problem-dependent regret in terms of the minimum 
gap \citep{Lat15-ucb}. Specifically, it can happen that
\eqn{
\label{eq:moss}
R^{\text{moss}}_{\mu,i^*} \in \Omega\left(\frac{K}{\Delta_{\min}} \log n\right)\,, 
}
where $\Delta_{\min} = \min_{i : \Delta_i > 0} \Delta_i$.
On the other hand, the order-optimal asymptotic regret can be significantly smaller. Specifically, UCB by \cite{ACF02} satisfies
\eqn{
\label{eq:ucb}
R^{\text{ucb}}_{\mu,i^*} \in O\left(\sum_{i : \Delta_i > 0} \frac{1}{\Delta_i} \log n\right)\,,
}
which for unequal gaps can be much smaller than \cref{eq:moss} and is asymptotically order-optimal \citep{LR85}. 
The problem is that MOSS explores only enough to obtain minimax regret, but sometimes obtains
minimax regret even when a more conservative algorithm would do better. It is worth remarking that this effect is harder to observe than one might think.
The example given in the afforementioned technical report is carefully tuned to exploit this failing, but still 
requires $n = 10^9$ and $K = 10^3$ before significant problems arise.
In all other experiments MOSS was performing admirably in comparison to UCB.

All these problems can be avoided by modifying UCB rather than MOSS. The cost is a factor of $O(\sqrt{\log n})$.
The algorithm is similar to \cref{alg:moss}, but chooses the action that maximises the following index.
\eq{
I_t = \argmax_i \hat \mu_{i,T_i(t-1)} + \sqrt{\frac{(2 + \epsilon) \log t}{T_i(t-1)}} - \sqrt{\frac{\log n}{n_i}}\,,
}
where $\epsilon > 0$ is a fixed arbitrary constant. 
\begin{theorem}\label{thm:u-ucb}
If $\pi$ is the strategy of unbalanced UCB with $n_i = n^2 / B_i^2$ and $B \in \calB$, then the regret of the unbalanced UCB satisfies: 
\begin{enumerate}
\item (problem-independent regret). $R^\pi_{\mu,i^*} \in O\left(B_{i^*} \sqrt{\log n}\right)$.
\item (problem-dependent regret). Let $A = \set{i : \Delta_i \geq 2\sqrt{1/n_{i^*} \log n}}$. Then
\eq{
R^\pi_{\mu,i^*} \in O\left(B_{i^*} \sqrt{\log n} \ind{A \neq \emptyset} + \sum_{i \in A} \frac{1}{\Delta_i} \log n\right)\,.
}
\end{enumerate}
\end{theorem}
\ifsup
The proof may be found in \cref{app:thm:u-ucb}.
\fi
The indicator function in the problem-dependent bound vanishes for sufficiently large $n$ provided $n_{i^*} \in \omega(\log(n))$, which is equivalent to 
$B_{i^*} \in o(n / \sqrt{\log n})$.
Thus for reasonable choices of $B_1,\ldots,B_{K}$ the algorithm is going to enjoy the same asymptotic performance as UCB.
\cref{thm:u-ucb} may be proven for any index-based algorithm for which it can be shown that
\eq{
\E T_i(n) \in O\left(\frac{1}{\Delta_i^2} \log n\right)\,,
}
which includes (for example) KL-UCB \citep{CGMMS13} and Thompson sampling (see analysis by \cite{AG12,AG12b} and original paper by \cite{Tho33}), but not OC-UCB \citep{Lat15-ucb} or MOSS \citep{AB09}.

\ifsup
\subsection*{A Note on Constants}

The constants in the statement of \cref{thm:moss} can be improved by carefully tuning all thresh-holds, but the proof would
grow significantly and I would not expect a corresponding boost in practical performance. In fact, the reverse is true, since the ``weak'' bounds used
in the proof would propagate to the algorithm. 
Also note that the $4$ appearing in the square root
of the unbalanced MOSS algorithm is due to the fact that I am not assuming rewards are 
bounded in $[0,1]$ for which the variance is at most $1/4$. It is possible to replace the $4$ with $2 + \epsilon$ for any $\epsilon > 0$ by
changing the base in the peeling argument in the proof of \cref{lem:peel} as was done by \cite{Bub10} and others.
\fi

\subsection*{Experimental Results}

\pgfplotstableread[comment chars={\%}]{data/fig1.txt}{\tableTwo}
\pgfplotstableread[comment chars={\%}]{data/fig2.txt}{\tableTen}
\pgfplotsset{cycle list={{blue}, {red,dotted}, {green!50!black,dashdotdotted}, {black}}}
\pgfplotsset{every axis plot/.append style={line width=1.5pt}}

\newcommand{\defaultaxis}{
        xlabel shift=-5pt,
        ylabel shift=-2pt,
        width=8.5cm,
        height=4.5cm,
        legend cell align=left,
        compat=newest}

I compare MOSS and unbalanced MOSS in two simple simulated examples, both with horizon $n = 5000$.
Each data point is an empirical average of $\sim\!\!10^4$ i.i.d.\ samples, so error bars are too small to see. Code/data is available in the supplementary material.
The first experiment has $K = 2$ arms and $B_1 = n^{\frac{1}{3}}$ and $B_2 = n^{\frac{2}{3}}$.
I plotted the results for $\mu = (0, -\Delta)$ for varying $\Delta$. As predicted, the new algorithm performs significantly better than MOSS
for positive $\Delta$, and significantly worse otherwise (\cref{fig:exp1}).
The second experiment has $K = 10$ arms. This time $B_1 = \sqrt{n}$ and $B_k = (k-1)H \sqrt{n}$ with $H = \sum_{k=1}^9 1/k$.
Results are shown for $\mu_k = \Delta \ind{k = i^*}$ for $\Delta \in [0,1/2]$ and $i^* \in \set{1,\ldots,10}$.
Again, the results agree with the theory. The unbalanced algorithm is superior to MOSS for $i^* \in \set{1,2}$ and inferior otherwise (\cref{fig:exp2}).

\begin{minipage}{5.8cm}
\begin{figure}[H]
  \centering
  \hspace{-0.4cm}
  \begin{tikzpicture}[font=\scriptsize]
    \begin{axis}[\defaultaxis,
        xmin=-0.5,
        xmax=0.5,
        ymin=0,
        width=5.8cm,
        xlabel={$\Delta$},
        ylabel={Regret}]
      \addplot+[] table[x index=0,y index=1] \tableTwo;
      \addlegendentry{MOSS};

      \addplot+[] table[x index=0,y index=2] \tableTwo;
      \addlegendentry{U.\ MOSS};
    \end{axis}
  \end{tikzpicture}
  \caption{}\label{fig:exp1}
\end{figure}
\end{minipage}
\hspace{-1.2cm}
\begin{minipage}{10cm}
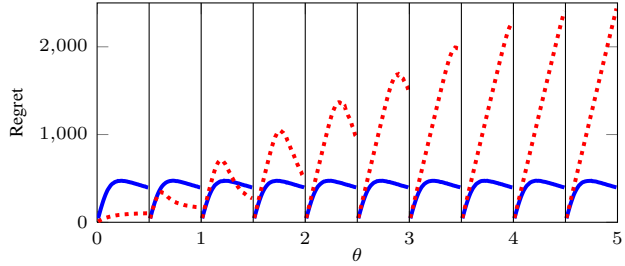
\begin{figure}[H]
  \centering
  \begin{tikzpicture}[font=\scriptsize]
    \begin{axis}[\defaultaxis,
        xmin=0,
        xmax=5,
        ymin=0,
        ymax=2500,
        xlabel={$\theta$},
        ylabel={Regret}]
      \addplot+[blue,solid,mark=none,restrict x to domain=0:0.5] table[x index=0,y index=1] \tableTen;
      \addplot+[blue,solid,mark=none,restrict x to domain=0.5:1] table[x index=0,y index=1] \tableTen;
      \addplot+[blue,solid,mark=none,restrict x to domain=1:1.5] table[x index=0,y index=1] \tableTen;
      \addplot+[blue,solid,mark=none,restrict x to domain=1.5:2] table[x index=0,y index=1] \tableTen;
      \addplot+[blue,solid,mark=none,restrict x to domain=2:2.5] table[x index=0,y index=1] \tableTen;
      \addplot+[blue,solid,mark=none,restrict x to domain=2.5:3] table[x index=0,y index=1] \tableTen;
      \addplot+[blue,solid,mark=none,restrict x to domain=3:3.5] table[x index=0,y index=1] \tableTen;
      \addplot+[blue,solid,mark=none,restrict x to domain=3.5:4] table[x index=0,y index=1] \tableTen;
      \addplot+[blue,solid,mark=none,restrict x to domain=4:4.5] table[x index=0,y index=1] \tableTen;
      \addplot+[blue,solid,mark=none,restrict x to domain=4.5:5] table[x index=0,y index=1] \tableTen;

      \addplot+[red,mark=none,dotted,restrict x to domain=0:0.5] table[x index=0,y index=2] \tableTen;
      \addplot+[red,mark=none,dotted,restrict x to domain=0.5:1] table[x index=0,y index=2] \tableTen;
      \addplot+[red,mark=none,dotted,restrict x to domain=1:1.5] table[x index=0,y index=2] \tableTen;
      \addplot+[red,mark=none,dotted,restrict x to domain=1.5:2] table[x index=0,y index=2] \tableTen;
      \addplot+[red,mark=none,dotted,restrict x to domain=2:2.5] table[x index=0,y index=2] \tableTen;
      \addplot+[red,mark=none,dotted,restrict x to domain=2.5:3] table[x index=0,y index=2] \tableTen;
      \addplot+[red,mark=none,dotted,restrict x to domain=3:3.5] table[x index=0,y index=2] \tableTen;
      \addplot+[red,mark=none,dotted,restrict x to domain=3.5:4] table[x index=0,y index=2] \tableTen;
      \addplot+[red,mark=none,dotted,restrict x to domain=4:4.5] table[x index=0,y index=2] \tableTen;
      \addplot+[red,mark=none,dotted,restrict x to domain=4.5:5] table[x index=0,y index=2] \tableTen;

      \addplot[thin,samples=50,smooth] coordinates {(0.5,0)(0.5,2500)};
      \addplot[thin,samples=50,smooth] coordinates {(1.5,0)(1.5,2500)};
      \addplot[thin,samples=50,smooth] coordinates {(2.5,0)(2.5,2500)};
      \addplot[thin,samples=50,smooth] coordinates {(3.5,0)(3.5,2500)};
      \addplot[thin,samples=50,smooth] coordinates {(4.5,0)(4.5,2500)};
      \addplot[thin,samples=50,smooth] coordinates {(1,0)(1,2500)};
      \addplot[thin,samples=50,smooth] coordinates {(2,0)(2,2500)};
      \addplot[thin,samples=50,smooth] coordinates {(3,0)(3,2500)};
      \addplot[thin,samples=50,smooth] coordinates {(4,0)(4,2500)};
      \addplot[thin,samples=50,smooth] coordinates {(5,0)(5,2500)};
    \end{axis}
  \end{tikzpicture}
  \caption{$\theta = \Delta + (i^*-1) / 2$}\label{fig:exp2}
\end{figure}
\end{minipage}

Sadly the experiments serve only to highlight the plight of the biased learner, which suffers significantly worse results than
its unbaised counterpart for most actions.

\section{Discussion}\label{sec:conc}

I have shown that the cost of favouritism for multi-armed bandit algorithms is rather serious. If an algorithm exhibits a small worst-case
regret for a specific action, then the worst-case regret of the remaining actions is necessarily significantly larger than the well-known 
uniform worst-case bound of $\Omega(\sqrt{Kn})$. This unfortunate result is in stark contrast to the experts setting for which there exist
algorithms that suffer constant regret with respect to a single expert at almost no cost for the remainder.
Surprisingly, the best achievable (non-uniform) worst-case bounds are determined up to a permutation almost entirely by the value of the smallest worst-case regret.

There are some interesting open questions. Most notably, in the adversarial setting I am not sure if the upper or lower bound is tight (or neither). 
It would also be nice to know if the constant factors can be determined exactly asymptotically, but so far this has not been done even
in the uniform case. For the stochastic setting it is natural to ask if the OC-UCB algorithm can also be modified. Intuitively one would expect this to be
possible, but it would require re-working the very long proof.

\subsubsection*{Acknowledgements}
I am indebted to the very careful reviewers who made many suggestions for improving this paper. Thank you!

\newpage
\appendix

\bibliographystyle{plainnat}
\bibliography{all}

\begin{thebibliography}{19}
\providecommand{\natexlab}[1]{#1}
\providecommand{\url}[1]{\texttt{#1}}
\expandafter\ifx\csname urlstyle\endcsname\relax
  \providecommand{\doi}[1]{doi: #1}\else
  \providecommand{\doi}{doi: \begingroup \urlstyle{rm}\Url}\fi

\bibitem[Agrawal and Goyal(2012{\natexlab{a}})]{AG12}
Shipra Agrawal and Navin Goyal.
\newblock Further optimal regret bounds for thompson sampling.
\newblock In \emph{Proceedings of International Conference on Artificial
  Intelligence and Statistics (AISTATS)}, 2012{\natexlab{a}}.

\bibitem[Agrawal and Goyal(2012{\natexlab{b}})]{AG12b}
Shipra Agrawal and Navin Goyal.
\newblock Analysis of thompson sampling for the multi-armed bandit problem.
\newblock In \emph{Proceedings of Conference on Learning Theory (COLT)},
  2012{\natexlab{b}}.

\bibitem[Audibert and Bubeck(2009)]{AB09}
Jean-Yves Audibert and S{\'e}bastien Bubeck.
\newblock Minimax policies for adversarial and stochastic bandits.
\newblock In \emph{COLT}, pages 217--226, 2009.

\bibitem[Auer et~al.(1995)Auer, Cesa-Bianchi, Freund, and Schapire]{ACFS95}
Peter Auer, Nicolo Cesa-Bianchi, Yoav Freund, and Robert~E Schapire.
\newblock Gambling in a rigged casino: The adversarial multi-armed bandit
  problem.
\newblock In \emph{Foundations of Computer Science, 1995. Proceedings., 36th
  Annual Symposium on}, pages 322--331. IEEE, 1995.

\bibitem[Auer et~al.(2002)Auer, Cesa-Bianchi, and Fischer]{ACF02}
Peter Auer, Nicol{\'o} Cesa-Bianchi, and Paul Fischer.
\newblock Finite-time analysis of the multiarmed bandit problem.
\newblock \emph{Machine Learning}, 47:\penalty0 235--256, 2002.

\bibitem[Boucheron et~al.(2013)Boucheron, Lugosi, and Massart]{BLM13}
Stephane Boucheron, Gabor Lugosi, and Pascal Massart.
\newblock \emph{Concentration Inequalities: A Nonasymptotic Theory of
  Independence}.
\newblock OUP Oxford, 2013.

\bibitem[Bubeck(2010)]{Bub10}
S{\'e}bastien Bubeck.
\newblock \emph{Bandits games and clustering foundations}.
\newblock PhD thesis, Universit{\'e} des Sciences et Technologie de Lille-Lille
  I, 2010.

\bibitem[Bubeck and Cesa-Bianchi(2012)]{BC12}
S\'ebastien Bubeck and Nicol\`o Cesa-Bianchi.
\newblock \emph{Regret Analysis of Stochastic and Nonstochastic Multi-armed
  Bandit Problems}.
\newblock Foundations and Trends in Machine Learning. Now Publishers
  Incorporated, 2012.
\newblock ISBN 9781601986269.

\bibitem[Capp{\'e} et~al.(2013)Capp{\'e}, Garivier, Maillard, Munos, and
  Stoltz]{CGMMS13}
Olivier Capp{\'e}, Aur{\'e}lien Garivier, Odalric-Ambrym Maillard, R{\'e}mi
  Munos, and Gilles Stoltz.
\newblock Kullback--{L}eibler upper confidence bounds for optimal sequential
  allocation.
\newblock \emph{The Annals of Statistics}, 41\penalty0 (3):\penalty0
  1516--1541, 2013.

\bibitem[Cesa-Bianchi(2006)]{Ces06}
Nicolo Cesa-Bianchi.
\newblock \emph{Prediction, learning, and games}.
\newblock Cambridge University Press, 2006.

\bibitem[Even-Dar et~al.(2008)Even-Dar, Kearns, Mansour, and Wortman]{EKMW08}
Eyal Even-Dar, Michael Kearns, Yishay Mansour, and Jennifer Wortman.
\newblock Regret to the best vs. regret to the average.
\newblock \emph{Machine Learning}, 72\penalty0 (1-2):\penalty0 21--37, 2008.

\bibitem[Hutter and Poland(2005)]{HP05}
Marcus Hutter and Jan Poland.
\newblock Adaptive online prediction by following the perturbed leader.
\newblock \emph{The Journal of Machine Learning Research}, 6:\penalty0
  639--660, 2005.

\bibitem[Kapralov and Panigrahy(2011)]{KP11}
Michael Kapralov and Rina Panigrahy.
\newblock Prediction strategies without loss.
\newblock In \emph{Advances in Neural Information Processing Systems}, pages
  828--836, 2011.

\bibitem[Koolen(2013)]{Koo13}
Wouter~M Koolen.
\newblock The pareto regret frontier.
\newblock In \emph{Advances in Neural Information Processing Systems}, pages
  863--871, 2013.

\bibitem[Lai and Robbins(1985)]{LR85}
Tze~Leung Lai and Herbert Robbins.
\newblock Asymptotically efficient adaptive allocation rules.
\newblock \emph{Advances in applied mathematics}, 6\penalty0 (1):\penalty0
  4--22, 1985.

\bibitem[Lattimore(2015)]{Lat15-ucb}
Tor Lattimore.
\newblock Optimally confident {UCB} : Improved regret for finite-armed bandits.
\newblock Technical report, 2015.
\newblock URL \url{http://arxiv.org/abs/1507.07880}.

\bibitem[Liu and Li(2015)]{LL15}
Che-Yu Liu and Lihong Li.
\newblock On the prior sensitivity of thompson sampling.
\newblock \emph{arXiv preprint arXiv:1506.03378}, 2015.

\bibitem[Sani et~al.(2014)Sani, Neu, and Lazaric]{SNL14}
Amir Sani, Gergely Neu, and Alessandro Lazaric.
\newblock Exploiting easy data in online optimization.
\newblock In \emph{Advances in Neural Information Processing Systems}, pages
  810--818, 2014.

\bibitem[Thompson(1933)]{Tho33}
William Thompson.
\newblock On the likelihood that one unknown probability exceeds another in
  view of the evidence of two samples.
\newblock \emph{Biometrika}, 25\penalty0 (3/4):\penalty0 285--294, 1933.

\end{thebibliography}

\ifsup
\section{Table of Notation}\label{app:notation}
\noindent
\hspace{-0.2cm}
\begin{tabular}{p{2cm}p{10cm}}
$n$ & time horizon \\
$K$ & number of available actions \\
$t$ & time step \\
$k,i$ & actions \\
$\calB$ & set of achievable worst-case regrets defined in \cref{def:B} \\
$\frontier$ & boundary of $\calB$ \\
$\mu$ & vector of expected rewards $\mu \in [0,1]^K$ \\
$\mu^*$ & expected return of optimal action \\
$\Delta_j$ & $\mu^* - \mu_j$ \\
$\Delta_{ji}$ & $\mu_i - \mu_j$ \\
$\pi$ & bandit strategy \\
$I_t$ & action chosen at time step $t$ \\
$R_{\mu,k}^\pi$ & regret of strategy $\pi$ with respect to the $k$th arm \\
$R_k^\pi$ & worst-case regret of strategy $\pi$ with respect to the $k$th arm \\
$\hat \mu_{k,s}$ & empirical estimate of the return of the $k$ action after $s$ samples \\
$T_k(t)$ & number of times action $k$ has been taken at the end of time step $t$ \\
$i^*$ & optimal action \\
$\log_{+}\!(x)$ & maximum of $0$ and $\log(x)$ \\ 
$\mathcal N(\mu, \sigma^2)$ & Gaussian with mean $\mu$ and variance $\sigma^2$ \\
\end{tabular}
\fi

\ifsup
\section{Proof of \cref{thm:u-ucb}}\label{app:thm:u-ucb}

Recall that the proof of UCB depends on showing that
\eq{
\E T_i(n) \in O\left(\frac{1}{\Delta_i^2} \log n\right)\,.
}
Now unbalanced UCB operates exactly like UCB, but with shifted rewards.
Therefore for unbalanced UCB we have
\eq{
\E T_i(n) \in O\left(\frac{1}{\bar \Delta_i^2} \log n\right)\,,
}
where
\eq{
\bar \Delta_i \geq \Delta_i + \sqrt{\frac{\log n}{n_i}} - \sqrt{\frac{\log n}{n_{i^*}}}\,.
}
Define :
\eq{
A &= \set{i : \Delta_i \geq 2\sqrt{\frac{\log n}{n_{i^*}}}} \\
}
If $i \in A$, then $\Delta_i \leq 2 \bar \Delta_i$ and $\bar \Delta_i \geq \sqrt{\frac{\log n}{n_i}}$.
Therefore
\eq{
\Delta_i \E T_i(n)
&\in O\left(\frac{\Delta_i}{\bar \Delta_i^2} \log n\right) 
\subseteq O\left(\frac{1}{\bar \Delta_i} \log n\right) 
\subseteq O\left(\sqrt{n_i \log n}\right)
\subseteq O\left(\frac{n}{B_i} \sqrt{\log n}\right)\,.
}
For $i \notin A$ we have $\Delta_i < 2\sqrt{\frac{\log n}{n_{i^*}}}$ thus
\eq{
\E\left[\sum_{i \notin A} \Delta_iT_i(n) \right] 
&\in O\left(n \sqrt{\frac{\log n}{n_{i^*}}}\right)
\subseteq O\left(B_{i^*} \sqrt{\log n}\right)\,.
}
Therefore
\eq{
R^\pi_{\mu,i^*} 
= \sum_{i=1}^K \Delta_i \E T_i(n) 
\in O\left(\left(B_{i^*} + \sum_{i \in A} \frac{n}{B_i}\right) \sqrt{\log n}\right)
&= O\left(B_{i^*} \sqrt{\log n}\right)
}
as required.
For the problem-dependent bound we work similarly.
\eq{
R^\pi_{\mu,i^*}
&= \sum_{i=1}^K \Delta_i \E T_i(n) \\
&\in O\left(\sum_{i \in A} \frac{1}{\bar \Delta_i} \log n + \ind{A \neq \emptyset} B_{i^*}\sqrt{\log n}\right) \\
&\in O\left(\sum_{i \in A} \frac{1}{\Delta_i} \log n + \ind{A \neq \emptyset} B_{i^*}\sqrt{\log n}\right)\,.
}

\fi

\ifsup
\section{KL Techniques}\label{app:kl}

Let $\mu_1, \mu_k \in \R^K$ be two bandit environments as defined in the proof of \cref{thm:lower}.
Here I prove the claim that
\eq{
\E^\pi_{\mu_k} T_k(n) - \E^\pi_{\mu_1} T_k(n) \leq n\epsilon_k \sqrt{\E^\pi_{\mu_1} T_k(n)}\,.
}
The result follows along the same lines as the proof of the lower bounds given by \cite{ACFS95}.
Let $\set{\calF_t}_{t=1}^n$ be a filtration where $\calF_t$ contains information about rewards and actions chosen
up to time step $t$. So $g_{I_t,t}$ and $\ind{I_t = i}$ are measurable with respect to $\calF_t$.
Let $P_1$ and $P_k$ be the measures on $\calF$ induced by bandit problems $\mu_1$ and $\mu_k$ respectively. 
Note that $T_k(n)$ is a $\calF_n$-measurable random variable bounded in $[0,n]$. Therefore
\eq{
\E^\pi_{\mu_k} T_k(n) - \E^\pi_{\mu_1} T_k(n) 
&\sr{(a)}\leq n \sup_A \left|P_1(A) - P_2(A) \right| \\
&\sr{(b)}\leq n \sqrt{\frac{1}{2} \KL(P_1, P_k)}\,, 
}
where the supremum in (a) is taken over all measurable sets (this is the total variation distance) and (b) follows from
Pinsker's inequality.
It remains to compute the KL divergence. 
Let $P_{1,t}$ and $P_{k,t}$ be
the conditional measures on the $t$th reward. By the chain rule for the KL divergence we have
\eq{
\KL(P_1, P_k) = \sum_{t=1}^n \E_{P_1} \KL(P_{1,t}, P_{k,t})
\sr{(a)}= 2\epsilon_k^2 \sum_{t=1}^n \E_{P_1} \ind{I_t = k} 
= 2\epsilon_k^2 \E^\pi_{\mu_1} T_k(n)\,,
}
where (a) follows by noting that if $I_t \neq k$, then the distribution of the rewards at time step $t$ is the same for both 
bandit problems $\mu_1$ and $\mu_k$. For $I_t = k$ we have the difference in means is $(\mu_k)_k - (\mu_1)_k = \epsilon_k$ and
since the distributions are Gaussian the KL divergence is $2\epsilon_k^2$.
For Bernoulli random noise the KL divergence is also $\Theta(\epsilon_k^2)$ provided $(\mu_k)_k \approx (\mu_1)_k \approx 1/2$ and so
a similar proof works for this case. See the work by \cite{ACFS95} for an example.

\fi

\ifsup
\section{Adversarial Bandits}\label{app:adv}

In the adversarial setting I obtain something similar. First I introduce some new notation.
Let $g_{i,t} \in [0,1]$ be the gain/reward from choosing action $i$ at time step $t$. This is chosen in an arbitrary way by the adversary with $g_{i,t}$
possibly even dependent on the actions of the learner up to time step $t$. The regret difference between the gains obtained by the learner and those
of the best action in hindsight. 
\eq{
R_g^\pi = \max_{i \in \set{1,\ldots,K}} \E\left[\sum_{t=1}^n g_{i,t} - g_{I_t,t}\right]\,.
}
I make the most obvious modification to the Exp3-$\gamma$ algorithm, which is to bias the prior towards the special action and tune the learning rate accordingly.
The algorithm accepts as input the prior $\rho \in [0,1]^K$, which must satisfy $\sum_i \rho_i = 1$, and the learning rate $\eta$.

\begin{center}
\begin{minipage}{11cm}
\begin{algorithm}[H]
\caption{Exp3-$\gamma$}\label{alg:ewa}
\begin{algorithmic}[1]
\State {\bf Input:} $K$, $\rho \in [0,1]^K$, $\eta$
\State $w_{i,0} = \rho_i$ for each $i$
\For{$t \in 1,\ldots,n$}
\State Let $p_{i,t} = \frac{w_{i,t-1}}{\sum_{i=1}^K w_{i,t-1}}$
\State Choose action $I_t = i$ with probability $p_{i,t}$ and observe gain $g_{I_t,t}$
\State $\tilde \ell_{t,i} = \frac{(1 - g_{t,i}) \ind{I_t = i}}{p_{i,t}}$
\State $w_{i,t} = w_{i,t-1} \exp\left(-\eta \tilde \ell_{t,i}\right)$
\EndFor
\end{algorithmic}
\end{algorithm}
\end{minipage}
\end{center}

The following result follows trivially from the standard proof.

\todot{Does Exp3-$\gamma$ even appear in this reference? First reference for Exp-$\gamma$?}

\begin{theorem}[\cite{BC12}]\label{thm:ewa}
Let $\pi$ be the strategy determined by \cref{alg:ewa}, then
\eq{
R_g^\pi \leq \eta Kn + \frac{1}{\eta} \log\frac{1}{\rho_{i^*}}\,.
}
\end{theorem}

\begin{corollary}\label{cor:ewa}
If $\rho$ is given by
\eq{
\rho_i = \begin{cases}
\exp\left(-\frac{B_1^2}{4Kn}\right) & \text{if } i = 1 \\
(1 - \rho_1) / (K-1) &\text{otherwise}
\end{cases}
}
and $\eta = B_1 / (2Kn)$, then
\eq{
R_g^\pi \leq \begin{cases}
B_1 & \text{if } i^* = 1 \\
\frac{B_1}{2} + \frac{2Kn}{B_1} \log\left(\frac{4Kn(K-1)}{B_1^2}\right) & \text{otherwise}\,.
\end{cases}
}
\end{corollary}

\begin{proof}
The proof follows immediately from \cref{thm:ewa} by noting that for $i^* \neq 1$ we have
\eq{
\log \frac{1}{\rho_{i^*}} 
&= \log \left(\frac{K-1}{1 - \exp\left(-\frac{B_1^2}{4Kn}\right)}\right) \\
&\leq \log \left(\frac{4Kn(K-1)}{B_1^2}\right) 
}
as required.
\end{proof}

\fi

\ifsup
\section{Concentration}\label{app:conc}

The following straight-forward concentration inequality is presumably well known and the proof of an almost identical result is available by \cite{BLM13}, but
an exact reference seems hard to find.

\begin{theorem}\label{thm:conc}
Let $X_1,X_2,\ldots, X_n$ be independent and $1$-subgaussian, then
\eq{
\P{\exists t \leq n : \frac{1}{t} \sum_{s \leq t} X_s \geq \frac{\epsilon}{t}} \leq \exp\left(-\frac{\epsilon^2}{2n}\right)\,.
}
\end{theorem}

\begin{proof}
Since $X_i$ is $1$-subgaussian, by definition it satisfies
\eq{
(\forall \lambda \in \R) \qquad \E\left[\exp\left(\lambda X_i\right)\right] \leq \exp\left(\lambda^2/2\right)\,.
}
Now $X_1,X_2,\ldots$ are independent and zero mean, so by convexity of the exponential function $\exp(\lambda \sum_{s=1}^t X_s)$ is a sub-martingale.
Therefore if $\epsilon > 0$, then by Doob's maximal inequality 
\eq{
\label{eq:maximal}
\P{\exists t \leq n : \sum_{s=1}^t X_s \geq \epsilon} 
&= \inf_{\lambda \geq 0} \P{\exists t \leq n : \exp\left(\lambda \sum_{s=1}^t X_s\right) \geq \exp\left(\lambda \epsilon\right)}  \\
&\leq \inf_{\lambda \geq 0} \exp\left(\frac{\lambda^2 n}{2} -\lambda \epsilon\right)  \\
&= \exp\left(-\frac{\epsilon^2}{2n}\right)
}
as required.
\end{proof}
\else
\fi

\end{document}